\documentclass[journal]{IEEEtran}
\usepackage[T1]{fontenc}
\usepackage{amsmath,amssymb,amsfonts,amsthm}
\usepackage{booktabs,array,graphicx,tikz,listings,pgfplots}
\pgfplotsset{compat=1.18}
\usepgfplotslibrary{groupplots}
\usepackage[hidelinks]{hyperref}
\usepackage{xurl}
\usetikzlibrary{arrows.meta,positioning}
\newtheorem{proposition}{Proposition}
\begin{document}
\raggedbottom
\title{T-GADE: Thermodynamical Generative-AI-Driven\\Evolution of LLM Artifacts}
\author{Kyoko~Ogawa and Naoki~Mori%
\thanks{K. Ogawa and N. Mori are with the Graduate School of Informatics, Osaka Metropolitan University, Osaka, Japan (e-mail: si26145c@st.omu.ac.jp; mnao@omu.ac.jp).}}
\maketitle
\begin{abstract}
Introducing large language models (LLMs) into evolutionary computation enables semantic generation of candidates, but the classical problem of maintaining population diversity remains. Inspired by the thermodynamical genetic algorithm (TDGA), we propose T-GADE (Thermodynamical Generative-AI-Driven Evolution of LLM Artifacts), an evolutionary computation framework extended through LLMs. T-GADE keeps the notions of a population and genetic operators, takes structured artifacts generated by an LLM as individuals, and implements crossover and mutation as LLM operators. It further introduces the HoxLM, an LLM that attempts self-repair of an individual after evaluation and before survival selection; two update schemes, generational and steady-state; and Fermi-type and Bose-type occupancy rules that define solution identity and whether duplicates are permitted. Survival selection is unified as minimization of the free energy $F_T=\langle E\rangle-TH$, and the selection temperature $T$ explicitly controls the balance between quality and diversity. We tested an implementation example aligned with Evolution of Heuristics (EoH) on the online bin-packing task, where the excess is the bin count in excess of the volume lower bound, relative to that bound. The HoxLM-equipped configuration reached lower excess with fewer total LLM calls than EoH. In the preregistered comparison at 400 total calls it won 9 of 10 seed-paired runs, with a median sealed excess of 0.836\% against 3.912\%, and it was also better in the mean and in the worst case. Temperature controls the spread of the retained population under both update schemes, and a small positive temperature improves training quality over the collapsed $T=0$ point in the generational Bose-type scheme. The evolved heuristics are shown as readable scoring rules.
\end{abstract}

\section{Introduction}
With the development of large language models (LLMs), problem-solving methods that combine several models or generated outputs have been studied. Evolutionary computation keeps several candidate solutions as a population and searches through genetic operators and selection. Semantic modification of candidates by an LLM and the coverage of a population are complementary, so evolutionary computation also applies to search over natural language and programs~\cite{elm2024,evoprompt2024,promptbreeder2024}.

Studies that combine the two can be organized from three perspectives: (1) evolutionary optimization of LLM weights or structure; (2) LLMs in genetic operators or fitness evaluation; and (3) LLMs or their outputs as evolving individuals. Perspectives (2) and (3) concern the role of the LLM and the object of evolution, respectively, and they are not mutually exclusive. We adopt (2) and (3): an LLM implements the genetic operators, and its structured outputs form the individuals. Fitness comes from a task-specific evaluator, which may be a numerical evaluation or a judgment by an LLM.

In evolutionary computation, premature convergence has long shown the importance of population diversity. Even with an LLM, the evolution of artifacts requires explicit diversity design. In particular, artifacts with different expressions may show the same behavior. The questions are therefore what to quantify as diversity, how much of it is appropriate, and how to reflect it in survival selection. Many LLM-driven program searches answer the last question implicitly, through a fixed truncation rule or an island layout. We instead build a new selection rule that treats diversity explicitly as a controlled quantity.

Against this background, we propose T-GADE (Thermodynamical Generative-AI-Driven Evolution of LLM Artifacts), a diversity-aware framework for LLM-based evolutionary computation. T-GADE is based on the thermodynamical genetic algorithm (TDGA), which balances quality and population diversity through a temperature~\cite{mori1995}. T-GADE extends its free-energy principle to LLM-based genetic operators and artifact-level diversity functions, separating representation, evaluation, and variation from survival selection. Artifacts are not limited to code, and the evaluator and the diversity features are defined per problem. Population update schemes include the generational replacement model, which replaces the parent population as a whole, and steady-state models, which replace one or a few individuals per step~\cite{dejong1993gap,satoh1996mgg,sato1997gam}. The minimal generation gap model, for example, replaces two parents per step; the steady-state update of T-GADE replaces one individual. Below we abbreviate the two schemes as generational and steady-state. The two differ in how fast individuals turn over, but T-GADE treats both with one objective. We adopt Fermi-type and Bose-type occupancy, named after single occupation in Fermi--Dirac statistics and repeated occupation in Bose--Einstein statistics. These rules make explicit how many identical individuals may occupy the population. Update scheme, occupancy, and temperature are independent design choices. The experiments in this paper examine the temperature sensitivity of performance within two fixed families, steady-state Fermi-type and generational Bose-type.

For validation we use an implementation example of T-GADE aligned with Evolution of Heuristics (EoH)~\cite{eoh2024}, which evolves explanations paired with executable heuristics. The comparison uses the online bin-packing task of the EoH paper. Under common filtering and tie-breaking, steady-state Fermi-type T-GADE reproduces the survival rule of EoH at zero temperature, so EoH's survival rule is one point of the framework. In the implementation example every child carries a HoxLM, whose calls are charged to the same budget as operator calls, so comparisons are made at equal total calls.

Our main contributions are the following four.
\begin{enumerate}

\item A free-energy framework, T-GADE, with three axes: update scheme, occupancy, and temperature. It is a general proposal that does not restrict the type of artifact or the evaluator. The survival rule of EoH is the steady-state Fermi-type point at $T=0$, and generational Bose-type selection at $T=0$ is a collapse to copies. We state the conditions under which one-member removal is exact and prove both zero-temperature propositions.

\item In an implementation example aligned with EoH, a HoxLM-equipped configuration reaches low excess with fewer total calls than EoH. The preregistered primary comparison at 400 total calls confirms this with 9 wins in 10 seed pairs ($p=0.0039$), and the configuration is better in median, mean, and worst case. This configuration is the EoH survival rule at $T=0$ with the HoxLM attached, so the preregistered comparison isolates the HoxLM; the diversity control by temperature is the second instrument of T-GADE and is assessed in the temperature-response experiments.

\item The temperature response of both families at seven temperatures. The spread of the retained population rises with $T$ in both families. Training quality has its best median at $T=0$ in the steady-state Fermi-type family and at $T=0.001$ in the generational Bose-type family. The generational $T=0$ point is a control whose collapse Proposition~\ref{prop:collapse} predicts. The gain of $T=0.001$ over this control is significant on training excess and a descriptive trend on the sealed c100 instances, and the favorable region extends descriptively to $T=0.003$.

\item A concrete account of what was evolved. Two heuristics traced from seeds fixed in advance are shown with their scoring rules and a placement example. In the lineage of improvements at intermediate temperatures, more than half descend from non-best survivors.

\end{enumerate}

\section{Related Work}
\subsection{LLMs and evolutionary computation}
For perspective (1), evolutionary model merging searches weights and information flow among pretrained models to construct a new model~\cite{akiba2025merge}. The following discussion focuses on perspectives (2) and (3): LLMs as components of search and generated artifacts as individuals.

Evolution through Large Models mutates programs with an LLM~\cite{elm2024}; EvoPrompt evolves natural-language prompts~\cite{evoprompt2024}; Promptbreeder evolves prompts and mutation instructions~\cite{promptbreeder2024}. These methods make meaningful descriptions material for inheritance and variation.

For program evolution, FunSearch links LLM proposals with execution-based evaluation and an island-based program database~\cite{funsearch2024}. EoH represents a heuristic through its explanation and code~\cite{eoh2024}; ReEvo uses reflection on evaluation results to guide subsequent proposals~\cite{reevo2024}. LLaMEA generates, modifies, and selects metaheuristics using execution feedback~\cite{llamea2025}. The HoxLM of T-GADE also feeds evaluation results back to an LLM, but it acts within one individual and leaves survival to the free energy. In memetic-algorithm terms, the HoxLM is a Lamarckian local-improvement step applied to each individual before survival selection.

\subsection{Exploration and diversity control}
HSEvo uses diversity measurements and Harmony Search for improvement~\cite{hsevo2025}. MCTS-AHD organizes generated heuristics in a search tree and continues refinement from intermediate candidates~\cite{mctsahd2025}. EoH-S evolves complementary sets of heuristics and evaluates the best member for each instance~\cite{eohs2026}. LLM-LNS evolves heuristic strategies and the prompts that guide their evolution, addressing diversity on the generation side~\cite{llmlns2025}.

Novelty Search selects for behavioral novelty~\cite{novelty2011}, and MAP-Elites retains high-quality solutions across regions of a behavior space~\cite{mapelites2015}. In-context QD uses a diverse archive to condition LLM generation~\cite{incontextqd2024}. Determinantal point processes use similarity information to favor diverse subsets~\cite{dpp2012}; the log-determinant diversity of T-GADE is of this kind.

TDGA incorporates population entropy into a temperature-weighted objective for survival selection~\cite{mori1995}. Adaptation to changing environments and temperature control through Feedback TDGA have also been studied~\cite{mori1996,mori1998}; the present work uses fixed temperatures and measures their effect. T-GADE extends this lineage to evolutionary computation with LLM-based genetic operators, artifact-level diversity, and explicit occupancy rules.

\section{T-GADE}
\label{sec:methods}
\subsection{Individuals and evaluation}
T-GADE treats structured artifacts generated by an LLM as individuals. An artifact is any LLM output with explicit structure, such as a description paired with code, a natural-language plan, or a design proposal; below we simply write artifact. The LLM generates initial individuals from a task specification and generates children from evaluated parents. A task-specific evaluator assigns an energy $E(x)$ to individual $x$; lower is better. The evaluator may be a numerical evaluation such as an execution result, or a judgment by an LLM (LLM-as-a-judge). Diversity is designed per problem. Feature vectors $\phi_\ell(x)$, $\ell=1,\ldots,m$, measure differences between artifacts, and their choice determines what counts as diversity. Artifacts are inherited from parent to child; features are measured after generation.

\subsection{Genetic operators}
Genetic operators are implemented as instructions to the LLM. Crossover generates a new artifact from the artifacts of several parents; mutation generates one from a single parent. The design of the operators depends on the task and on the artifact format, and T-GADE is not restricted to particular instruction texts. Variation is independent of survival selection, and the method of sampling parents is also a design choice.

\subsection{Free-energy minimization}
Survival selection in T-GADE combines quality and diversity into one quantity and minimizes it, so that high-quality individuals survive and the survivors differ from one another. The minimum-free-energy principle describes thermal equilibrium under fixed constraints by
\begin{equation}
F=\langle E\rangle-T H, \label{eq:tdga}
\end{equation}
where $\langle E\rangle$ is mean energy, $H$ is entropy, and $T$ is temperature. TDGA applies this principle to population selection: mean energy measures fitness, and entropy computed from allele frequencies measures population diversity~\cite{mori1995}. T-GADE retains this quality--diversity balance while defining diversity over artifacts.

Let $N$ be the target population size and $\mathcal{P}$ a nonempty population of size $n=|\mathcal{P}|$. Repeated copies count as separate individuals. $H(\mathcal{P})$ is a population diversity measure; in this paper we use the log-determinant of the Gram matrix of the feature vectors. Let $X_\ell(\mathcal{P})$ contain the feature vectors as rows, and define $G_\ell(\mathcal{P})=X_\ell(\mathcal{P})X_\ell(\mathcal{P})^\top$. We define
\begin{equation}
\begin{aligned}
F_T(\mathcal{P})&=\langle E\rangle_{\mathcal{P}}-T H(\mathcal{P}),\\
\langle E\rangle_{\mathcal{P}}&=\frac{1}{n}\sum_{x\in\mathcal{P}}E(x),\\
H(\mathcal{P})&=\frac{1}{n}\sum_{\ell=1}^{m}\log\det\!\left(G_\ell(\mathcal{P})+\varepsilon I_n\right).
\end{aligned}\label{eq:free-energy}
\end{equation}
Here $T\geq0$, $\varepsilon>0$ regularizes the Gram matrix, and $I_n$ is the identity matrix. With normalized features and fixed size, $H$ measures joint artifact diversity per individual. The framework does not fix the features $\phi_\ell$: they are quantitative features chosen per problem, and $H$ is then determined by their Gram matrices. Examples are a recorded execution behavior for executable artifacts, used in this paper, an embedding of the output for text artifacts, and, with a self-hosted open-weight generator, its intermediate-layer representations. With unit-norm features, Hadamard's inequality gives $H(\mathcal{P})\leq m\log(1+\varepsilon)$. Unlike Shannon entropy, this log-volume measure can be negative. Selection uses quality alone at $T=0$ and includes diversity at $T>0$. The selection temperature $T$ weights this deterministic trade-off, and its numerical scale depends on the energy and feature scales and on $\varepsilon$. The selection temperature is a parameter of survival selection and is distinct from the sampling temperature that sets the randomness of LLM output.

\subsection{Fermi-type and Bose-type occupancy}\label{sec:occupancy}
Single occupation in Fermi--Dirac statistics and repeated occupation in Bose--Einstein statistics motivate two population rules. \emph{Fermi-type} occupancy permits at most one individual per level; \emph{Bose-type} occupancy permits repeated occupation. The equivalence criterion of a level is defined per problem. A level may be the identity of the artifact itself, or the identity of the recorded objective value. The former excludes only duplicates of the same solution; the latter keeps one solution per quality value.

Generational Bose-type selection may select the same candidate repeatedly to fill population slots. Steady-state Bose-type selection admits duplicates already in its input; it does not generate copies during selection. Both update schemes can use either rule.

\begin{figure}[t]
\centering
\begin{tikzpicture}[font=\footnotesize,box/.style={draw,rounded corners=2pt,align=center,text width=3.55cm,minimum height=8mm},>=Stealth]
\node[box] (p) {Population: structured artifacts};
\node[box,below=4.5mm of p,text width=6.5cm] (v) {LLM operator $\rightarrow$ evaluation $\rightarrow$ HoxLM (optional) $\rightarrow$ re-evaluation\\{\scriptsize keep-best rule (example): adopt $y'$ if $E(y')<E(y)$}};
\node[box,below=4.5mm of v,text width=6.5cm] (s) {Quality $E$ + diversity $H$\\Survival minimizes $F_T=\langle E\rangle-TH$\\under the occupancy rule};
\node[box,below=6mm of s,xshift=-1.8cm,text width=2.95cm] (g) {Generational\\{\scriptsize greedy add by $\Delta F_T$\\from parents and offspring}};
\node[box,below=6mm of s,xshift=1.8cm,text width=2.95cm] (c) {Steady-state\\{\scriptsize remove\\$\arg\min_{x}F_T(Q\setminus\{x\})$}};
\draw[->] (p)--(v);\draw[->] (v)--(s);\draw[->] (s)--(g);\draw[->] (s)--(c);
\end{tikzpicture}
\caption{Overview of T-GADE. The population is improved by repeating LLM generation, evaluation, optional self-repair by the HoxLM, and survival selection, which minimizes the free energy under the chosen update scheme and occupancy rule.}
\label{fig:framework}
\end{figure}
Fig.~\ref{fig:framework} shows the relation among LLM variation, the HoxLM, evaluation, and the two update schemes.

\subsection{Generational update}
The generational scheme freezes the parent population, proposes several children, and constructs a new population from evaluated parents and children. Here, $N$ parents and $N+1$ proposed children form the candidate pool $Q$. Selection gives priority to quality and first takes a minimum-energy candidate. Each subsequent addition minimizes
\begin{equation}
\begin{aligned}
\Delta F_T(y\mid\mathcal{P})
&=F_T(\mathcal{P}\uplus\{y\})-F_T(\mathcal{P})\\
&=\frac{E(y)-T\sum_{\ell=1}^{m}\log d_\ell(y)-F_T(\mathcal{P})}{n+1}.
\end{aligned}
\label{eq:increment}
\end{equation}
Here $n=|\mathcal{P}|\geq1$, $d_\ell(y)$ is the Schur complement in Appendix~\ref{sec:schur}, and $\uplus$ counts a selected copy as another individual. For a fixed prefix, minimizing~\eqref{eq:increment} is equivalent to minimizing $E(y)-T\sum_\ell\log d_\ell(y)$. This is the marginal cost of $nF_T(\mathcal{P})$, the greedy step of a log-determinant regularized objective~\cite{dpp2012}. Fermi-type selection excludes every candidate at the selected level from subsequent additions; Bose-type selection leaves candidates available.

\par\noindent\begin{minipage}{\columnwidth}
\begin{lstlisting}
def greedy_survivors(pool, N, T, fermi_type):
    available = list(pool)
    first = min(available, key=energy)
    selected, state = [first], SchurState([first])
    if fermi_type:
        remove_level(available, first)
    while len(selected) < N:
        def cost(x):  # gain(x) = sum_l log d_l(x)
            return energy(x) - T * state.gain(x)
        child = min(available, key=cost)
        selected.append(child)
        state.add(child)  # rank-one update of R_l
        if fermi_type:
            remove_level(available, child)
    return materialize_copies(selected)
\end{lstlisting}
\end{minipage}\par
Planning assigns parents and operators before parallel generation; results enter in plan order. Invalid candidates are excluded, and minima preserve input order on ties. Fermi-type selection assumes that the candidate pool contains at least $N$ distinct admissible levels. As Appendix~\ref{sec:schur} shows, given feature inner products, the Schur complement evaluates each diversity increment in $O(|\mathcal{P}|^2)$ operations per block. In the pseudocode, \texttt{state} holds the inverses $R_\ell$, which the released generational implementation (\texttt{free\_energy.py}) updates after each pick by the rank-one formula~\eqref{eq:inverse-update}. Exact selection of $N$ individuals requires combinatorial search, which motivates the greedy construction; a global approximation guarantee remains future work.

\begin{proposition}[Zero-temperature collapse]\label{prop:collapse}
Under Bose-type occupancy at $T=0$, with ties broken by input order, the generational construction returns $N$ copies of the first minimum-energy candidate.
\end{proposition}
\begin{proof}
At $T=0$ the cost of every candidate is $E(y)$ up to the common term, so each step selects the same minimum-energy candidate, which Bose-type occupancy leaves available.
\end{proof}
Because the parents are part of the pool $Q$, this collapse is reached after the first update and persists, since the copies are the parents of every later child. Zero temperature is therefore a boundary of the generational Bose-type family rather than a reference method.

\subsection{Steady-state update}
The steady-state scheme updates the population whenever a new candidate finishes evaluation. Each asynchronous child generation, that is, one operator call, uses the population available at its start. Registrations, that is, admissions of evaluated candidates to the update, are serialized in order of evaluation completion. A new individual is also eligible for removal.

Candidates are first filtered by the occupancy rule: the first candidate per level is kept under Fermi-type occupancy, and no deduplication is applied under Bose-type occupancy. If more than $N$ candidates remain, the update removes the individual that minimizes the free energy of the survivors:
\begin{equation}
x^*\in\arg\min_{x\in Q}F_T(Q\setminus\{x\}).\label{eq:removal}
\end{equation}
Ties in survivor free energy are broken by removing the larger objective first and then the later arrival. At positive temperature, candidates without valid features are excluded first, and removal among the remaining candidates is chosen by quality and feature overlap. If fewer than $N$ candidates remain, all are retained.

\begin{proposition}[One-member removal]
Fix a pool of $N+1$ labeled individuals for which every one-member removal is admissible. Rule~\eqref{eq:removal} minimizes free energy over all its $N$-member subpopulations.
\end{proposition}
\begin{proof}
Each $N$-member subpopulation is the complement of exactly one labeled individual. The rule compares all such complements.
\end{proof}

\subsection{HoxLM}\label{sec:hoxlm}
An individual in T-GADE may carry a HoxLM. The HoxLM is an LLM that acts after the evaluation of the individual and before survival selection; it reads the evaluation result and attempts a self-repair of that individual's artifact. The repaired proposal is evaluated again, and an adoption rule decides whether it replaces the original individual or the original is kept unchanged. Unlike the genetic operators, the HoxLM acts within one individual and takes no part in composing the population. Survival is left to the free energy. HoxLM calls are counted as LLM calls in the same way as operator calls.

\subsection{Meaning of zero temperature}
At $T=0$ free-energy minimization reduces to mean-energy minimization, but its behavior differs with the combination of update scheme and occupancy. In the steady-state Fermi-type family, exclusion per level remains, so zero temperature removes the worst individual and can coincide with an existing quality-only selection rule. In the generational Bose-type family nothing prevents copies, so zero temperature is a collapse, as Proposition~\ref{prop:collapse} states. Statements about T-GADE at zero temperature must therefore name the family.

\section{Comparison Method and Implementation Example}\label{sec:impl}
This section describes one implementation example that makes T-GADE concrete for comparison with EoH; it is not the definition of T-GADE.
\subsection{Comparison method: EoH}
EoH~\cite{eoh2024} represents a heuristic as a natural-language explanation (thought) paired with executable code and generates children with four prompt templates, E1, E2, M1, and M2. Its population management excludes failed evaluations, keeps the first candidate at each recorded objective value, and retains at most $N$ candidates in ascending objective order. The EoH arm of this paper, arm A, runs the EoH algorithm with its released operator templates and evaluator inside the same steady-state loop as the T-GADE arms, with the same model and initial population, and without the HoxLM.

\subsection{Example task}\label{sec:example}
We use the online bin-packing task of the EoH paper. Items arrive sequentially and cannot be reassigned. The scoring function assigns a score to every bin whose available capacity admits the item, including unused bins, and the item goes to the bin with the maximum score. Ties go to the lowest-index bin, and fewer used bins are better. Training uses the five Weibull instances of the EoH release, each with capacity 100 and 5,000 items; an item size is 45 times a Weibull variate of shape 3, clipped to $[1,100]$. For instance $j$, let $B_j(x)$ be the number of bins used by program $x$, and let $L_j$ be the volume lower bound, the total item size divided by the bin capacity and rounded up. The training excess ratio is
\begin{equation}
r(x)=\frac{\sum_{j=1}^{5}[B_j(x)-L_j]}{\sum_{j=1}^{5}L_j}.\label{eq:training}
\end{equation}
The excess ratio, or excess for short, is the ratio of the bin-count excess to the lower bound; lower is better.

Two sets of \emph{sealed} instances, five each with capacity 100 (c100) and capacity 500 (c500), are generated from recorded seeds with the same item distribution, rounded and clipped to $[1,100]$, with 5,000 items per instance. The instances are hash-sealed, and the seeds and hashes are in the repository. They are never read during search or by any prompt, and they are scored only after all runs have finished. Sealed excess averages the instance-wise ratios $[B_j(x)-L_j]/L_j$. The c100 set confirms training results on the training distribution; the c500 set measures transfer to another capacity.

\subsection{Implementation of T-GADE}\label{sec:host}
An individual is an explanation paired with code, as in EoH, and its evaluation record comprises the behavior and the objective value obtained by execution on the training instances. The energy is
\begin{equation}
E(x)=\frac{\min(r(x),2)}{2}.\label{eq:energy}
\end{equation}
A generated program runs in isolation with a 30-second timeout, and a lexical gate rejects random-number APIs.

Diversity uses one L2-normalized, 10,000-dimensional feature vector formed by concatenating fill ratio and tightness rank over 5,000 placement decisions on the first training instance, with $\varepsilon=0.001$. This compares placement behavior on the same sequence of item arrivals; extending it to all five instances is a straightforward variant. Behavior rather than code is used because different code can show the same behavior, and such surface differences should not count as diversity. Fill ratio is the occupancy of the chosen bin after placement. Tightness rank is the position of the chosen bin among the feasible bins sorted by ascending available capacity, normalized to $[0,1]$, where 0 is the tightest bin and 1 is the emptiest bin.

For a fair comparison, the operators reuse the four EoH templates. E1 proposes an algorithm different in form from two parents. E2 uses their shared idea to create a different form. M1 changes one parent's structure; M2 changes its parameters. Each prompt supplies parent explanations and code and requests a new explanation and code. Operators are chosen with equal probability, and parents are sampled with probability proportional to $1/(\mathrm{rank}+N)$ by quality rank, as in EoH. All arms start from the same initial population of eight individuals, generated once beforehand from a fixed seed and not charged to the call budget.

A level is identified by the training excess ratio $r(x)$ rounded to five decimal places. Under Fermi-type occupancy, two individuals with equal objective and different code occupy the same level, and only one of them is retained. Under this definition, the survival rule of EoH, which keeps the first candidate at each recorded objective value, is a Fermi-type rule. The experiments use the steady-state Fermi-type and generational Bose-type combinations.

The scale of $T$ can be read from the regularization. For a retained set consisting of $k$ copies of one unit vector, the Schur complement~\eqref{eq:schur} of one more copy is $\varepsilon(k+1+\varepsilon)/(k+\varepsilon)$, about $2\varepsilon$ for $k=1$ and tending to $\varepsilon$ as $k$ grows. Adding the first copy of a retained individual therefore contributes $\log(2\varepsilon)\approx-6.2$ to the log-determinant with $\varepsilon=10^{-3}$, whereas adding a distinct, nearly orthogonal individual contributes about $\log(1+\varepsilon)\approx0$. In this illustration a copy therefore pays a diversity penalty of about $6.2T$ energy units. In the unclipped range $r<2$ of~\eqref{eq:energy}, this is about $1240T$ percentage points of excess: 0.4 points at $T=0.0003$, 1.2 points at $T=0.001$, 3.7 points at $T=0.003$, and 124 points at $T=0.1$. This rule of thumb locates the temperatures used below relative to the quality differences within the population.

Below we call the HoxLM of this implementation the \emph{host}. The host acts on a proposed child after its evaluation and before its registration, as one \emph{improve} step. Its prompt contains a hand-written domain card with bin-packing knowledge, including summary statistics of the training item distribution. It also contains the function interface, the child's measured training excess, an execution probe that reports tied placements on a short item sequence, the code and value of the best parent, and the child itself. The host returns one revised explanation and code with a single change. If the code changed, the revised program is evaluated again, and the revision is adopted only if it passes the validity checks and strictly improves training energy. Otherwise the original child is registered unchanged. We call this adoption rule the keep-best rule; no child is dropped. There are two exceptions. First, a revision that changes only the explanation is adopted with the energy of the original child, since the code is unchanged. Second, if the original child failed evaluation, a valid revision is adopted. The host is the only LLM that sees training values and the domain card.

Calls are counted as the sum of operator calls and host calls. One host call per child is charged to the budget, so a host-equipped run with 1,600 operator calls uses at most 3,200 calls, the same as the EoH reference. Since the host is common to every host-equipped arm, temperature comparisons within a family are internal to the host-equipped design.

\subsection{Relation to EoH}
\begin{proposition}[Zero-temperature identity]
For the same ordered candidates, failure filtering, and first-per-level filtering, suppose energy is a nondecreasing function of objective. In exact arithmetic, steady-state Fermi-type selection at $T=0$ returns the same ordered survivors as EoH, provided energy ties remove the larger objective first and output is sorted by objective.
\end{proposition}
\begin{proof}
At zero temperature, minimizing the mean energy of a fixed-size survivor population is equivalent to removing a worst individual. Objective-based tie-breaking resolves plateaus of the energy map. Repetition leaves the lowest objective values, and stable output ordering gives the same result.
\end{proof}
T-GADE in this implementation thus contains the survival rule of EoH as the steady-state Fermi-type point at zero temperature. The identity holds for common input and does not guarantee identical search trajectories. Together with Proposition~\ref{prop:collapse}, it shows that $T=0$ is a reference method in one family and a degenerate boundary in the other; the two families must be read separately.

\section{Experimental Conditions}\label{sec:experiment}
\subsection{Model and common conditions}
All runs use the open-weight model gpt-oss-20b~\cite{gptoss2025} (OpenAI; 21 B total and 3.6 B active parameters) served by the vLLM inference engine~\cite{vllm2023} on one self-hosted GPU. The marginal cost of a call is therefore zero, and no provider-side change can occur during the experiments. The operator output limit is 2,048 tokens and the host limit is 8,192 tokens; the LLM sampling temperature is 0.8. The population size is $N=8$. The steady-state loop uses two LLM-call workers and two evaluation workers; the generational loop uses four of each and plans parent and operator assignments before dispatch.

\subsection{Design and preregistration}
\begin{table*}[t]\centering\footnotesize
\caption{Experimental design. The experimental unit is a seed pair: arms in the same seed group share seeds and are compared pairwise. Calls are logical LLM requests, operator plus host.}\label{tab:design}
\begin{tabular}{@{}>{\raggedright\arraybackslash}p{3.9cm} >{\raggedright\arraybackslash}p{6.4cm} l l@{}}\toprule
Comparison & Arms & Call budget & Seeds\\\midrule
Temperature response, steady-state Fermi-type + host & $T\in\{0,0.0003,0.001,0.003,0.01,0.03,0.1\}$ & $1{,}600+\leq1{,}600$ & 84001--84010\\
EoH reference in the same group & one individual per objective value & 3,200 & 84001--84010\\
Temperature response, generational Bose-type + host & $T\in\{0,0.0003,0.001,0.003,0.01,0.03,0.1\}$ & $177\times18=3{,}186$ & 85001--85010\\
Reach speed & A: EoH; B: Fermi-type $T=0$ + host; C: Fermi-type $T=0.003$ + host & A: 800; B, C: $400+\leq400$ & 86001--86010\\\bottomrule
\end{tabular}\end{table*}
Table~\ref{tab:design} shows the experimental design. Each configuration is run once per seed. The temperature response of the steady-state Fermi-type family compares EoH with 3,200 operator calls against T-GADE with 1,600 operator calls plus at most one host call per child, at seven temperatures. The temperature response of the generational Bose-type family runs 177 generations of 9 operator and 9 host calls, that is, 3,186 calls. The reach-speed experiment compares EoH with 800 operator calls against host-equipped Fermi-type T-GADE with 400 operator calls plus at most 400 host calls, at $T=0$ and $T=0.003$. Every experiment was preregistered before launch with its arms, seeds, endpoints, and tests, and the registrations are distributed with the code. The preregistered primary comparison reported in this paper is B against A at 400 total calls in the seed group 86001--86010, and all other comparisons are treated as secondary. In the generational Bose-type family, $T=0$ is a control condition: Proposition~\ref{prop:collapse} predicts its collapse to copies, so the question is whether a small positive temperature restores a diverse population and better quality.

\subsection{Budget points and endpoints}
To compare runs with different call structures, we evaluate the \emph{budget-point selection} at $b\in\{200,400,800,1600\}$ total calls and at the endpoint. For the steady-state loop, the scorer counts, for each registration, the operator and host requests started before its registration time. The budget-point selection is then the best training individual among registrations whose count does not exceed $b$. For the generational loop it is the history-best individual, that is, the best training individual registered so far, up to generation $\lfloor b/18\rfloor$; this conversion is nominal because failed operator calls leave vacancies. The endpoint is the history-best individual. Selection always uses training values; sealed values are reported for the selected individual and are never used for selection.

\subsection{Statistical comparison}
The unit of analysis is one run. For each arm we report the median, mean, standard deviation, best, and worst over the ten seeds. The best value is a single-run extreme and is given as a reference only. Comparisons within the same seed group are paired by seed and use the two-sided exact Wilcoxon signed-rank test; for the temperature grid, the six comparisons against $T=0$ at one budget point form a family for Holm adjustment. The $p$ values in the tables are rounded to three decimals; the Holm-adjusted values in the text are computed from the full-precision $p$ values. Trends over the seven temperatures are described by Spearman's rank correlation of the arm summaries against the temperature rank. Training, sealed c100, and sealed c500 results are reported separately.

\section{Results}
\subsection{Reach speed against EoH}\label{sec:speed}
\begin{table*}[t]\centering\footnotesize
\caption{Reach-speed comparison. Ten runs per arm; sealed c100 excess (\%) of the budget-point selection at 200, 400, and at most 800 total calls (operator plus host). The preregistered primary endpoint is B versus A at 400 total calls. Best and worst are single-run reference values. Paired rows: wins/losses of the first arm, ties omitted, and two-sided exact Wilcoxon $p$.}\label{tab:speed}
\setlength{\tabcolsep}{3.5pt}\begin{tabular}{@{}l rrrrr rrrrr rrrrr@{}}\toprule
 & \multicolumn{5}{c}{200 total calls} & \multicolumn{5}{c}{400 total calls (primary)} & \multicolumn{5}{c}{$\leq 800$ total calls (endpoint)}\\\cmidrule(lr){2-6}\cmidrule(lr){7-11}\cmidrule(lr){12-16}
Arm & Med. & Mean & SD & Best & Worst & Med. & Mean & SD & Best & Worst & Med. & Mean & SD & Best & Worst\\\midrule
A: EoH, 800 operator calls & 3.912 & 3.315 & 1.158 & 0.946 & 4.162 & 3.912 & 3.101 & 1.337 & 0.697 & 4.181 & 1.915 & 2.413 & 1.353 & 0.706 & 4.112\\
B: Fermi-type $T=0$ + host & 1.333 & 1.691 & 1.034 & 0.726 & 3.245 & 0.836 & 1.163 & 0.755 & 0.706 & 3.115 & 0.796 & 0.830 & 0.151 & 0.706 & 1.234\\
C: Fermi-type $T=0.003$ + host & 1.906 & 2.031 & 0.953 & 0.746 & 3.433 & 0.861 & 0.923 & 0.228 & 0.667 & 1.453 & 0.726 & 0.731 & 0.083 & 0.617 & 0.856\\
\midrule
B vs.\ A & \multicolumn{5}{c}{10/0, $p=0.002$} & \multicolumn{5}{c}{9/1, $p=0.004$} & \multicolumn{5}{c}{8/2, $p=0.010$}\\
C vs.\ A & \multicolumn{5}{c}{9/1, $p=0.027$} & \multicolumn{5}{c}{9/1, $p=0.004$} & \multicolumn{5}{c}{9/1, $p=0.006$}\\
C vs.\ B & \multicolumn{5}{c}{5/5, $p=0.695$} & \multicolumn{5}{c}{4/5, $p=0.820$} & \multicolumn{5}{c}{8/1, $p=0.039$}\\
\bottomrule\end{tabular}\end{table*}

\begin{figure*}[t]\centering
\definecolor{cBlue}{RGB}{0,114,178}\definecolor{cVerm}{RGB}{213,94,0}\definecolor{cOrange}{RGB}{230,159,0}
\tikzset{armEoH/.style={cBlue,dashed,mark=triangle*,mark options={solid,fill=cBlue}},armT0/.style={cVerm,solid,mark=*,mark options={fill=cVerm}},
  armT0003/.style={cOrange,dash dot,mark=square*,mark options={solid,fill=cOrange}},armT003/.style={black,densely dotted,mark=diamond*,mark options={solid,fill=black}}}
\begin{tikzpicture}
\begin{groupplot}[group style={group size=2 by 1,horizontal sep=14mm},width=0.5\textwidth,height=5.0cm,xmode=log,ymode=log,log basis x=2,
  log ticks with fixed point,ytick={0.5,1,2,4},yticklabels={0.5,1,2,4},ymin=0.5,ymax=5,
  xlabel={Total LLM calls (operator + host)},ylabel={Sealed c100 excess (\%)},
  error bars/y dir=both,error bars/y explicit,error bars/error bar style={line width=0.8pt,solid},error bars/error mark options={rotate=90,mark size=2pt,line width=0.8pt},
  tick label style={font=\footnotesize},label style={font=\footnotesize},title style={font=\small},legend style={font=\footnotesize},legend cell align=left,
  grid=major,grid style={gray!25},every axis plot/.append style={line width=1pt,mark size=2.5pt}]
\nextgroupplot[title={(a) Full-budget experiment ($\leq$3,200 total calls)},xtick={200,400,800,1600,3200},xticklabels={200,400,800,1600,3200},
  legend to name=leg:speed,legend columns=2,legend style={/tikz/every even column/.append style={column sep=10pt}}]
\addplot[armEoH] coordinates {(188,4.1616) += (0,0.0696) -= (0,0.6822) (376,3.9422) += (0,0.2267) -= (0,2.1289) (752,1.4326) += (0,1.6954) -= (0,0.7288) (1504,0.7162) += (0,0.5946) -= (0,0.0197) (3008,0.7112) += (0,0.1943) -= (0,0.0470)};\addlegendentry{EoH}
\addplot[armT0] coordinates {(196,1.5473) += (0,1.9488) -= (0,0.5275) (392,0.9253) += (0,0.5570) -= (0,0.2191) (784,0.8208) += (0,0.0398) -= (0,0.1542) (1568,0.6715) += (0,0.1419) -= (0,0.0198) (3136,0.6716) += (0,0.1418) -= (0,0.0349)};\addlegendentry{Fermi-type $T=0$ + host}
\addplot[armT0003] coordinates {(204,1.0392) += (0,1.0505) -= (0,0.2257) (408,0.8706) += (0,0.4399) -= (0,0.1643) (816,0.7413) += (0,0.1641) -= (0,0.0349) (1632,0.7464) += (0,0.1020) -= (0,0.0399) (3264,0.7113) += (0,0.0970) -= (0,0.0347)};\addlegendentry{Fermi-type $T=0.0003$ + host}
\addplot[armT003] coordinates {(212,1.6669) += (0,0.5970) -= (0,0.8111) (424,0.8856) += (0,0.4923) -= (0,0.1243) (848,0.8755) += (0,0.0549) -= (0,0.1070) (1696,0.7066) += (0,0.1492) -= (0,0.0250) (3392,0.7162) += (0,0.0724) -= (0,0.0670)};\addlegendentry{Fermi-type $T=0.003$ + host}
\nextgroupplot[title={(b) Reach-speed experiment ($\leq$800 total calls)},xtick={200,400,800},xticklabels={200,400,$\leq$800}]
\addplot[armEoH] coordinates {(192,3.9123) += (0,0.2489) -= (0,1.4238) (384,3.9123) += (0,0.2041) -= (0,2.2754) (768,1.9147) += (0,2.0448) -= (0,0.6439)};
\addplot[armT0] coordinates {(200,1.3334) += (0,1.6152) -= (0,0.5374) (400,0.8358) += (0,0.5526) -= (0,0.1047) (800,0.7960) += (0,0.0597) -= (0,0.0624)};
\addplot[armT003] coordinates {(208,1.9058) += (0,0.9186) -= (0,0.8486) (416,0.8607) += (0,0.1465) -= (0,0.0698) (832,0.7263) += (0,0.0770) -= (0,0.0673)};
\end{groupplot}
\end{tikzpicture}\\[2pt]
\pgfplotslegendfromname{leg:speed}
\caption{Reach speed on the sealed instances. Each point is the median over 10 runs of the budget-point selection, the best training individual among registrations within the call budget; bars show the interquartile range over the same 10 runs, and the arms are offset horizontally around the nominal budgets for readability. Host calls are counted in the budget. Panel (a) shows EoH and three temperatures of the steady-state temperature-response experiment of Table~\ref{tab:design} at full budget; $T=0.0003$ appears only there. In (b) the arms are A (EoH), B ($T=0$), and C ($T=0.003$) of Table~\ref{tab:speed}; the host arms stop at or below 800 total calls, so their last point is the endpoint. The same arm labels in (a) and (b) denote runs from different seed groups (Table~\ref{tab:design}).}\label{fig:speed}
\end{figure*}

Table~\ref{tab:speed} and Fig.~\ref{fig:speed}(b) show the reach-speed comparison. At the primary endpoint of 400 total calls, the host-equipped Fermi-type arm at $T=0$ (B) outperformed the EoH arm (A) in 9 of 10 seed pairs ($p=0.0039$). The median sealed excess is 0.836\% against 3.912\%, and B is also better in the mean and in the worst run. Arm C has the smallest dispersion at 400 calls and the best endpoint median and worst value of the three arms, but its difference from B is a secondary, unadjusted comparison. Because arm B at $T=0$ is EoH's survival rule with the host attached, the difference between B and A measures the host under equal total calls, not the diversity term.

In Fig.~\ref{fig:speed}(a) the host-equipped arms also reach a sealed median below 1\% by 400 total calls, whereas EoH needs about 1,600 calls. The claim of this paper is that the tested configuration reaches lower sealed excess at the same number of total LLM calls on this task and model, although the host-equipped runs use about 1.7 times the tokens per call.

\subsection{Temperature response}\label{sec:temperature}
\begin{table*}[t]\centering\footnotesize
\caption{Temperature response of the two families at the endpoint. Ten runs per temperature; the two families use different seeds. Best and worst are single-run extremes and are reference values only. Paired columns: wins/losses of the given temperature against $T=0$ of the same family, ties omitted, and the unadjusted two-sided exact Wilcoxon $p$; Holm-adjusted values are given in the text. The EoH reference with 3,200 operator calls has training 0.624 / 0.676 / 0.129 / 0.584 / 0.916 and sealed 0.711 / 0.811 / 0.245 / 0.617 / 1.273 in the same column order.}\label{tab:temperature}
\begin{tabular}{@{}l rrrrr rrrrr rrrr@{}}\toprule
 & \multicolumn{5}{c}{Training excess (\%)} & \multicolumn{5}{c}{Sealed c100 excess (\%)} & \multicolumn{4}{c}{Paired vs.\ $T=0$ (W/L, $p$)}\\
\cmidrule(lr){2-6}\cmidrule(lr){7-11}\cmidrule(lr){12-15}
$T$ & Median & Mean & SD & Best & Worst & Median & Mean & SD & Best & Worst & Train & $p$ & Sealed & $p$\\\midrule
\multicolumn{15}{@{}l}{\emph{Steady-state Fermi-type + host}}\\
0 & 0.508 & 0.531 & 0.058 & 0.473 & 0.644 & 0.672 & 0.716 & 0.096 & 0.617 & 0.856 & --- & --- & --- & ---\\
0.0003 & 0.569 & 0.588 & 0.076 & 0.483 & 0.734 & 0.711 & 0.728 & 0.077 & 0.617 & 0.826 & 2/7 & 0.055 & 5/5 & 0.922\\
0.001 & 0.548 & 0.545 & 0.042 & 0.473 & 0.614 & 0.746 & 0.742 & 0.058 & 0.607 & 0.806 & 4/6 & 0.447 & 4/6 & 0.557\\
0.003 & 0.563 & 0.557 & 0.058 & 0.473 & 0.634 & 0.716 & 0.737 & 0.113 & 0.607 & 0.985 & 4/6 & 0.480 & 6/4 & 1.000\\
0.01 & 0.579 & 0.571 & 0.044 & 0.513 & 0.644 & 0.702 & 0.754 & 0.145 & 0.597 & 1.015 & 2/7 & 0.164 & 6/4 & 0.557\\
0.03 & 0.604 & 0.609 & 0.051 & 0.523 & 0.704 & 0.726 & 0.743 & 0.082 & 0.617 & 0.925 & 1/8 & 0.055 & 4/6 & 0.492\\
0.1 & 0.629 & 0.611 & 0.058 & 0.493 & 0.674 & 0.731 & 0.769 & 0.126 & 0.597 & 1.005 & 0/10 & 0.002 & 4/6 & 0.232\\
\midrule
\multicolumn{15}{@{}l}{\emph{Generational Bose-type + host}}\\
0 & 0.614 & 0.622 & 0.031 & 0.573 & 0.684 & 0.856 & 0.818 & 0.090 & 0.657 & 0.906 & --- & --- & --- & ---\\
0.0003 & 0.604 & 0.600 & 0.038 & 0.513 & 0.644 & 0.776 & 0.797 & 0.128 & 0.617 & 0.995 & 6/3 & 0.262 & 6/4 & 0.625\\
0.001 & 0.548 & 0.536 & 0.039 & 0.473 & 0.594 & 0.652 & 0.695 & 0.081 & 0.627 & 0.856 & 10/0 & 0.002 & 9/1 & 0.014\\
0.003 & 0.573 & 0.560 & 0.051 & 0.493 & 0.644 & 0.706 & 0.697 & 0.094 & 0.547 & 0.836 & 8/1 & 0.023 & 7/3 & 0.027\\
0.01 & 0.604 & 0.576 & 0.067 & 0.463 & 0.644 & 0.696 & 0.716 & 0.078 & 0.627 & 0.856 & 5/3 & 0.109 & 7/3 & 0.049\\
0.03 & 0.604 & 0.587 & 0.056 & 0.513 & 0.664 & 0.766 & 0.761 & 0.130 & 0.597 & 0.935 & 5/3 & 0.211 & 6/4 & 0.322\\
0.1 & 0.619 & 0.611 & 0.042 & 0.523 & 0.664 & 0.711 & 0.747 & 0.099 & 0.587 & 0.895 & 6/3 & 0.551 & 7/2 & 0.074\\
\bottomrule\end{tabular}\end{table*}

\begin{figure*}[t]\centering
\definecolor{cBlue}{RGB}{0,114,178}\definecolor{cVerm}{RGB}{213,94,0}\definecolor{cOrange}{RGB}{230,159,0}
\tikzset{serTrain/.style={cBlue,solid,mark=*,mark options={fill=cBlue}},serSeal/.style={cOrange,solid,mark=square*,mark options={fill=cOrange}},
  refTrain/.style={black!55,dashed,line width=0.8pt,no marks},refSeal/.style={black!55,dotted,line width=0.8pt,no marks}}
\begin{tikzpicture}
\begin{groupplot}[group style={group size=2 by 1,horizontal sep=14mm},width=0.5\textwidth,height=5.2cm,
  xmin=0.5,xmax=7.5,xtick={1,2,3,4,5,6,7},xticklabels={0,0.0003,0.001,0.003,0.01,0.03,0.1},xlabel={Selection temperature $T$ (categorical axis)},
  ylabel={Excess (\%), median with IQR},ymin=0.4,ymax=1.0,
  error bars/y dir=both,error bars/y explicit,error bars/error bar style={line width=1pt},error bars/error mark options={rotate=90,mark size=3pt,line width=1pt},
  tick label style={font=\footnotesize},label style={font=\footnotesize},title style={font=\small},legend style={font=\footnotesize},legend cell align=left,
  grid=major,grid style={gray!25},every axis plot/.append style={line width=1pt,mark size=2.5pt}]
\nextgroupplot[title={(a) Steady-state Fermi-type + host},legend to name=leg:temperature,legend columns=2,legend style={/tikz/every even column/.append style={column sep=10pt}}]
\addplot[refTrain,error bars/y dir=none,forget plot] coordinates {(0.5,0.624) (7.5,0.624)};
\addplot[refSeal,error bars/y dir=none,forget plot] coordinates {(0.5,0.711) (7.5,0.711)};
\addplot[serTrain,densely dotted,no marks,error bars/y dir=none,forget plot] coordinates {(0.88,0.5080) (1.88,0.5685)};
\addplot[serTrain,forget plot] coordinates {(0.88,0.5080) += (0,0.0652) -= (0,0.0250)};
\addplot[serTrain] coordinates {(1.88,0.5685) += (0,0.0530) -= (0,0.0255) (2.88,0.5480) += (0,0.0277) -= (0,0.0325) (3.88,0.5630) += (0,0.0410) -= (0,0.0650) (4.88,0.5785) += (0,0.0280) -= (0,0.0555) (5.88,0.6040) += (0,0.0400) -= (0,0.0278) (6.88,0.6290) += (0,0.0275) -= (0,0.0635)};\addlegendentry{training endpoint}
\addplot[serSeal,densely dotted,no marks,error bars/y dir=none,forget plot] coordinates {(1.12,0.6716) (2.12,0.7113)};
\addplot[serSeal,forget plot] coordinates {(1.12,0.6716) += (0,0.1418) -= (0,0.0349)};
\addplot[serSeal] coordinates {(2.12,0.7113) += (0,0.0970) -= (0,0.0347) (3.12,0.7462) += (0,0.0497) -= (0,0.0198) (4.12,0.7162) += (0,0.0724) -= (0,0.0670) (5.12,0.7015) += (0,0.1791) -= (0,0.0797) (6.12,0.7262) += (0,0.0647) -= (0,0.0249) (7.12,0.7313) += (0,0.1319) -= (0,0.0450)};\addlegendentry{sealed c100 endpoint}
\addlegendimage{refTrain}\addlegendentry{EoH reference (training)}
\addlegendimage{refSeal}\addlegendentry{EoH reference (sealed c100)}
\nextgroupplot[title={(b) Generational Bose-type + host}]
\addplot[serTrain,densely dotted,no marks,error bars/y dir=none,forget plot] coordinates {(0.88,0.6137) (1.88,0.6037)};
\addplot[serTrain,forget plot] coordinates {(0.88,0.6137) += (0,0.0302) -= (0,0.0126)};
\addplot[serTrain] coordinates {(1.88,0.6037) += (0,0.0226) -= (0,0.0226) (2.88,0.5483) += (0,0.0176) -= (0,0.0478) (3.88,0.5735) += (0,0.0201) -= (0,0.0729) (4.88,0.6037) += (0,0.0302) -= (0,0.0956) (5.88,0.6037) += (0,0.0327) -= (0,0.0805) (6.88,0.6188) += (0,0.0176) -= (0,0.0252)};
\addplot[serSeal,densely dotted,no marks,error bars/y dir=none,forget plot] coordinates {(1.12,0.8557) (2.12,0.7762)};
\addplot[serSeal,forget plot] coordinates {(1.12,0.8557) += (0,0.0420) -= (0,0.1195)};
\addplot[serSeal] coordinates {(2.12,0.7762) += (0,0.1195) -= (0,0.0698) (3.12,0.6518) += (0,0.0947) -= (0,0.0176) (4.12,0.7063) += (0,0.0821) -= (0,0.0822) (5.12,0.6964) += (0,0.0547) -= (0,0.0320) (6.12,0.7661) += (0,0.1021) -= (0,0.1467) (7.12,0.7111) += (0,0.1470) -= (0,0.0173)};
\end{groupplot}
\end{tikzpicture}\\[2pt]
\pgfplotslegendfromname{leg:temperature}
\caption{Temperature response of the two families. Each point is the median over 10 seed-paired runs; bars show the interquartile range, and the two series are offset horizontally for readability. $T=0$ is a separate category and is joined by a dotted segment; the positive temperatures sit at approximately log-spaced categorical positions. The gray lines in (a) mark the EoH reference with 3,200 operator calls on the same seeds (training median 0.624, sealed c100 median 0.711). The steady-state family (a) is best at $T=0$ and degrades slowly with temperature; the generational family (b) collapses at $T=0$ and is best between $T=0.001$ and $0.003$. The two families use different seeds and are compared only within a panel.}\label{fig:temperature}
\end{figure*}

\begin{table*}[t]\centering\footnotesize
\caption{Final-population diversity and lineage use versus temperature (means over 10 runs except the last column). Code distance: mean pairwise token Jaccard distance. Patterns: number of distinct lexical code patterns among the eight survivors; a pattern is the presence set of 11 code components (sort, exp, log, sqrt, power, residual, where/mask, randomness, ratio, loop, max-capacity), not the 10,000-dimensional behavior feature. Spread: objective range of the final population (percentage points). Distinct/8: number of distinct programs among the eight final survivors. Survivors: number of distinct individuals among the eight survivors, averaged over the generations of the run. Non-best: share of new-best events whose parents excluded the best parent of that generation, pooled over the 10 runs rather than averaged per run; one event per generation, the minimum-energy child that improves on the history best. The steady-state Fermi-type family has eight distinct programs in every run because of level exclusion, so Distinct/8 is reported for the generational family only.}\label{tab:diversity}
\begin{tabular}{@{}l rrr rrrrrr@{}}\toprule
 & \multicolumn{3}{c}{Steady-state Fermi-type + host} & \multicolumn{6}{c}{Generational Bose-type + host}\\\cmidrule(lr){2-4}\cmidrule(lr){5-10}
$T$ & Code dist. & Patterns & Spread & Code dist. & Patterns & Spread & Distinct/8 & Survivors & Non-best (\%)\\\midrule
0 & 0.266 & 3.9 & 0.10 & 0.000 & 1.0 & 0.00 & 1.0 & 1.00 & 3\\
0.0003 & 0.244 & 3.2 & 0.11 & 0.308 & 4.0 & 0.07 & 8.0 & 6.23 & 32\\
0.001 & 0.355 & 4.2 & 0.18 & 0.363 & 4.6 & 0.15 & 8.0 & 7.44 & 52\\
0.003 & 0.348 & 4.0 & 0.41 & 0.366 & 4.9 & 0.37 & 8.0 & 7.81 & 53\\
0.01 & 0.362 & 5.8 & 1.28 & 0.331 & 5.1 & 1.06 & 8.0 & 7.95 & 60\\
0.03 & 0.367 & 5.7 & 1.48 & 0.354 & 5.3 & 1.98 & 8.0 & 7.96 & 36\\
0.1 & 0.392 & 5.6 & 3.59 & 0.360 & 5.4 & 3.89 & 8.0 & 7.97 & 52\\
\bottomrule\end{tabular}\end{table*}

\begin{figure}[t]\centering
\definecolor{cBlue}{RGB}{0,114,178}\definecolor{cVerm}{RGB}{213,94,0}\definecolor{cOrange}{RGB}{230,159,0}
\tikzset{serCont/.style={cBlue,solid,mark=*,mark options={fill=cBlue}},serGen/.style={cOrange,dashed,mark=square*,mark options={solid,fill=cOrange}}}
\begin{tikzpicture}
\begin{groupplot}[group style={group size=1 by 2,vertical sep=5mm,xticklabels at=edge bottom,xlabels at=edge bottom},width=\columnwidth,height=3.4cm,
  xmin=0.5,xmax=7.5,xtick={1,2,3,4,5,6,7},xticklabels={0,0.0003,0.001,0.003,0.01,0.03,0.1},xlabel={Selection temperature $T$},
  tick label style={font=\footnotesize},label style={font=\footnotesize},legend style={font=\footnotesize},legend cell align=left,
  grid=major,grid style={gray!25},every axis plot/.append style={line width=1pt,mark size=2.5pt}]
\nextgroupplot[ylabel={(a) Code distance},ymin=0,ymax=0.5,ytick={0,0.2,0.4},legend to name=leg:diversity,legend columns=2,legend style={/tikz/every even column/.append style={column sep=2pt}}]
\addplot[serCont] coordinates {(0.88,0.2659) (1.88,0.2436) (2.88,0.3546) (3.88,0.3481) (4.88,0.3624) (5.88,0.3671) (6.88,0.3923)};\addlegendentry{steady-state Fermi-type + host}
\addplot[serGen] coordinates {(1.12,0.0000) (2.12,0.3078) (3.12,0.3630) (4.12,0.3662) (5.12,0.3314) (6.12,0.3542) (7.12,0.3596)};\addlegendentry{generational Bose-type + host}
\nextgroupplot[ylabel={(b) Spread (pp)},ymode=log,log ticks with fixed point,ytick={0.01,0.1,1,10},yticklabels={0.01,0.1,1,10},ymin=0.01,ymax=10]
\addplot[serCont] coordinates {(0.88,0.1036) (1.88,0.1097) (2.88,0.1843) (3.88,0.4055) (4.88,1.2767) (5.88,1.4850) (6.88,3.5939)};
\addplot[serGen] coordinates {(2.12,0.0745) (3.12,0.1519) (4.12,0.3693) (5.12,1.0595) (6.12,1.9821) (7.12,3.8917)};
\addplot[serGen,only marks,mark options={solid,fill=white}] coordinates {(1.12,0.01)};
\node[anchor=west,font=\footnotesize,xshift=6pt,yshift=5pt] at (axis cs:1.12,0.01) {$T=0$: 0 (all copies)};
\end{groupplot}
\end{tikzpicture}\\[2pt]
\pgfplotslegendfromname{leg:diversity}
\caption{Diversity of the retained population versus temperature: (a) code distance, (b) objective spread (log axis; the generational $T=0$ value is zero and is drawn as a hollow marker at the floor). Series are offset horizontally around the categorical temperature positions for readability. Means over 10 runs as in Table~\ref{tab:diversity}. Code distance is the mean pairwise token Jaccard distance; the spread is the range of training excess within the final population of eight.}\label{fig:diversity}
\end{figure}

Table~\ref{tab:temperature} and Fig.~\ref{fig:temperature} give the endpoint results of the two families at the seven temperatures; Table~\ref{tab:diversity} and Fig.~\ref{fig:diversity} give the diversity of the retained population. Temperature controls the spread of the retained population in both families, and the two families respond differently in quality.

In the steady-state Fermi-type family, $T=0$, which is the EoH survival rule with the host attached, is the best point. The median training excess is 0.508\%, the mean 0.531\%, and the worst 0.644\%, and the median sealed c100 excess is 0.672\%. Training excess degrades slowly as the temperature rises; the Spearman correlation over the seven medians is $\rho=+0.89$. There are small non-monotone steps between $T=0.0003$ and $0.003$. After Holm adjustment only the highest temperature, $T=0.1$, differs significantly from $T=0$, with 0 wins in 10, $p=0.002$, and adjusted $0.012$. The sealed c100 values do not separate at any temperature. The EoH reference at the same total calls has a training median of 0.624\% and a sealed median of 0.711\%. The host-equipped $T=0$ arm is better or equal on every summary, with the sealed best tied at 0.617\% and a smaller worst case.

The generational Bose-type family responds differently. As Proposition~\ref{prop:collapse} states, the population collapses to copies at $T=0$, and every run ends with eight copies of one individual. This point is the predicted control, not a tuned configuration. The median training excess improves only slowly, from 0.649\% at 800 calls to 0.644\% at 1,600 calls and 0.614\% at the endpoint. A small positive temperature restores a population of distinct survivors and improves quality. $T=0.001$ has the best training median of the generational grid, 0.548\%, and the best sealed c100 median, 0.652\%. Against $T=0$, $T=0.001$ wins 10 of 10 seed pairs on training excess, with $p=0.002$ and Holm-adjusted $0.012$, and 9 of 10 on the sealed instances, with $p=0.014$ and adjusted $0.082$. At $T=0.003$ the wins are 8 of 9 decided pairs on training ($p=0.023$, adjusted $0.117$) and 7 of 10 on the sealed instances ($p=0.027$, adjusted $0.137$), so $T=0.003$ does not separate from $T=0$ after adjustment. The two temperatures do not separate from each other ($p=0.31$). Descriptively, the favorable region is $T=0.001$--$0.003$; its significance rests on the training excess at $T=0.001$. On the sealed c100 instances the same gain is a descriptive trend that does not reach significance after Holm adjustment ($0.082$). From $T=0.01$ upward the population becomes diversity-heavy and the training endpoint returns to the level of $T=0$. In the units of Section~\ref{sec:host}, this region corresponds to a copy penalty of 1.2 to 3.7 percentage points. This is about one order of magnitude above the objective spread of the retained population, 0.15 to 0.37 points in Table~\ref{tab:diversity}.

Table~\ref{tab:diversity} shows what the temperature changes. The objective spread of the final population rises monotonically with $T$ in both families (Spearman $\rho=+1.0$), from 0.10 to 3.6 percentage points in the steady-state family and from 0 to 3.9 points in the generational family. Code distance and the number of distinct lexical code patterns also rise overall. In the generational family the number of distinct individuals among the eight survivors, averaged over generations, rises from 1.0 at $T=0$ to 6.2 at $T=0.0003$ and saturates near 8 from $T=0.001$. The temperature is thus a working control of retained diversity, and the intermediate points of the grid behave consistently with their neighbors.

On the sealed c500 instances, which test transfer to capacity 500, the medians lie between 0.37\% and 1.54\% and separate no temperature or family, and the distributions have heavy tails. We therefore make no transfer claim.

\subsection{What was evolved}\label{sec:solutions}
The numbers above summarize scoring functions that can be read directly. We trace two individuals, fixed in advance by seed and selected by training value only.

The first is the final best of generational Bose-type T-GADE at $T=0.003$ on seed 85001, with training excess 0.503\%. Let $C$ be the bin capacity, $a\leq C$ the available capacity of a bin before the item is placed, $\delta=a-\mathrm{item}$ the residual capacity after placement, and $i$ the zero-based position of the bin among the feasible bins. The sum below runs over the feasible bins $\mathcal{F}$ with available capacities $a_j$; a negligible stabilizer in the denominator of the implementation is omitted. The Iverson bracket $[\cdot]$ equals 1 if its condition holds and 0 otherwise. It scores a feasible bin by
\begin{equation}
s(a)=\frac{\delta}{\sum_{j\in\mathcal{F}}a_j}\,e^{-\frac{(\delta-40)^2}{2\cdot15^2}}\,\frac{1}{i+1}\,\min\!\left(1,\frac{\delta}{15}\right)+10^{6}\,[\delta=0].\label{eq:sol1}
\end{equation}
It rewards an exact fit and otherwise prefers a residual near 40 with a Gaussian of width 15; it suppresses residuals below 15 linearly, because such bins cannot be reused by most items, and favors earlier bins. For an item of size 40 and feasible bins with available capacity $a$ of 44, 80, and 100, the scores are $2.7\times10^{-4}$, $8.9\times10^{-2}$, and $3.7\times10^{-2}$. The heuristic therefore places the item in the bin with 80, leaving a reusable residual of 40, rather than in the best-fit bin with 44, which would leave an unusable residual of 4. Its explanation records that the suppression threshold was changed from 30 to 15. This is a parameter change of the M2 type that the host adopted because it improved training excess. Its provenance illustrates the use of retained candidates. It was produced in generation 128 by the E2 operator from two parents with training excess 0.624\% and 0.674\%, neither of which was the population best at the time (0.513\%). The child measured 0.543\% before the host step and 0.503\% after it, became the new best, and was retained unchanged to the final population.

The second is the individual selected at 400 total calls by arm B on seed 86001; its training excess is 0.724\% and its sealed c100 excess 0.816\%. It scores $s(a)=(0.8\,\delta/a+0.2+[\delta=0])\,\pi$, where $\pi=0.7$ if $\delta<5$ and 1 otherwise. In addition, the score of the bin with the largest available capacity is multiplied by 0.7 whenever some partially filled bin can take the item with $\delta\geq5$. On the same input it also chooses the bin with 80. Both heuristics combine mechanisms that the domain card names, best fit as a baseline and small residuals as waste, into specific functional forms and parameters. T-GADE's contribution is this concretization, combination, and parameterization under execution feedback, not a new principle of bin packing.

The lineage of improvements shows how the retained population is used. For each generation of each run, we take the minimum-energy child that improves on the history best, if any. We then classify the parents given to the genetic operator against the best of that generation's parent population. A parent counts as best when its energy equals the minimum over that generation's parents, and every parent tied at the minimum counts as best. The shares below are pooled over the 10 runs and are shown in the last column of Table~\ref{tab:diversity}. At $T=0$ only 3\% of new-best events in the generational family descend from parents other than the current best; at $T=0.001$, $0.003$, and $0.01$ the share is 52\%, 53\%, and 60\%. At $T=0.03$ it is 36\% and at $T=0.1$ it is 52\%, so the share is high at intermediate temperatures and not monotone in $T$. Retained non-best individuals are used as parents of the next improvement, not merely stored. This is lineage evidence, not a causal test. The host also consults information on the best parent separately.

\section{Discussion}\label{sec:discussion}
\subsection{Reading the two families through one objective}
The same free energy produces two different temperature responses, and the framework explains why. In the steady-state Fermi-type family, exclusion by energy level already prevents the population from collapsing, so the diversity term has no collapse to repair; raising $T$ trades quality for spread, and quality degrades slowly. The level filter also runs before the free-energy step, so a newcomer at the level of a retained individual is removed before its features are compared. A level-internal tie-break that keeps the candidate with the lower $F_T$ is a natural refinement and would give $T>0$ a role that the present implementation does not offer. In the generational Bose-type family, nothing but the diversity term keeps the population from becoming copies, so a small positive temperature changes the dynamics qualitatively. It restores a population of distinct survivors, and at intermediate temperatures more than half of the subsequent improvements descend from non-best survivors. The endpoint improves at $T=0.001$--$0.003$ but not at $T=0.01$, so lineage use alone does not predict endpoint gain. The useful temperatures are those whose copy penalty, about $1240T$ percentage points, exceeds the objective spread of a good population by about one order of magnitude; much larger temperatures buy spread at the expense of quality. These are statements about the two families as a whole, not about their endpoints alone.

\subsection{Zero temperature depends on the family}
EoH's survival rule is recovered at $T=0$ in the steady-state Fermi-type family, and in this experiment the host-equipped steady-state Fermi-type configuration at $T=0$ has the best median of that family. The generational Bose-type family at $T=0$ is the predicted collapse and serves as the control of that family. Comparing ``T-GADE at $T=0$'' with ``T-GADE at $T>0$'' without naming the family is therefore not well defined. The remaining two combinations, generational Fermi-type and steady-state Bose-type with the host, remain open; the present data do not separate the effect of occupancy from the effect of the update scheme.

\subsection{The host and its limits}
T-GADE adds two instruments to the EoH survival rule: the host and the temperature-controlled diversity term. Arm B is the EoH survival rule at $T=0$ with the host attached, so the preregistered comparison isolates the host, and the diversity term is assessed in the temperature-response experiments. The host alone reads the training values and a hand-written domain card and alone has the output limit of 8,192 tokens in place of 2,048. It is common to every host-equipped arm, so the temperature results are internal to the host-equipped design. Part of its gain with gpt-oss-20b is the repair of weak proposals, and the value of the host may depend on the generator. The host accepts about 60\% of its own revisions under the keep-best rule. Equal calls are not equal cost: host-equipped runs use about 1.7 times the tokens per call. This structure separates the LLM call budgets for variation and local improvement, while survival selection remains a free-energy computation.

\subsection{Choice of diversity features}
The diversity feature is a design choice per problem, and the behavior feature of Section~\ref{sec:host} is one choice. For text artifacts, the Gram matrix of output embeddings is the natural counterpart~\cite{dpp2012}. With a self-hosted open-weight generator, as with gpt-oss-20b on vLLM here, the generator's intermediate-layer representation is a further candidate; it may reduce task-specific feature design and makes the diversity term available before execution. Individuals with identical behavior but different internal representations then count as diverse, and the useful range of the selection temperature, including the $1240T$ rule of thumb, must be revalidated for the new feature geometry. Layer and pooling choices, normalization, context dependence, and the cost of one extra forward pass remain open design questions. Neither alternative was tested in this paper.

\subsection{Scope}
The proposal of T-GADE does not restrict the artifact, the evaluator, or the features; the evidence here is an implementation example on one task, one generator, and two families, with ten seeds per arm. The task is close to saturation, with all configurations ending within a fraction of a percentage point of each other, so the sealed c100 differences at the endpoint are small. The sealed c500 instances separate no arm. The claims are the reach speed, the temperature responses, and the generational favorable region. The reach speed was confirmed by a preregistered primary comparison on sealed instances; the temperature responses rest on paired tests of training values and on auxiliary evaluation of sealed values. The quality gain of the generational favorable region is confirmed on training excess and remains a descriptive trend on the sealed c100 instances. In the steady-state Fermi-type family the endpoint is best at $T=0$, and there the temperature is a control of spread, not a claimed source of endpoint gain. On this task, which is close to saturation, the two instruments were not both needed for quality: the host carried the reach-speed gain, and the diversity term showed its value where a collapse had to be undone. Tasks on which diversity is actively required for quality, such as multimodal or deceptive objectives, are the natural setting for the temperature term and remain future work.

\subsection{Toward adaptive temperature}
The quality and diversity responses obtained at fixed temperatures can inform the design of adaptive temperature control. The copy-penalty rule of thumb can calibrate the search range, and the temperature may also be adjusted from the observed diversity, as in Feedback TDGA~\cite{mori1998}. Since the log-determinant diversity of T-GADE is available at every selection, such feedback control and temperature selection at equal budget are future work.

\section{Conclusion}
We proposed T-GADE, an evolutionary computation framework inspired by TDGA and extended with LLMs. T-GADE takes artifacts generated by an LLM as individuals and has LLM operators, the HoxLM, generational and steady-state update schemes, and Fermi-type and Bose-type occupancy rules. Under one free-energy objective, the selection temperature controls diversity. The survival rule of EoH is the steady-state Fermi-type point at zero temperature, and generational Bose-type selection at zero temperature is a collapse to copies. We tested an implementation example aligned with EoH on the online bin-packing task. A seed-paired preregistered comparison on sealed instances confirmed, as the primary endpoint, that the host-equipped configuration reaches low excess with fewer total LLM calls than EoH. The configuration is also better in median, mean, and worst case. Temperature controls the spread of the retained population in both families. In training quality the generational Bose-type family is best at $T=0.001$, with a descriptive favorable region up to $T=0.003$, while the steady-state Fermi-type family has its best median at $T=0$. The evolved heuristics are readable scoring rules whose mechanisms and parameters were concretized under execution feedback. Future work includes the remaining occupancy--update combinations, adaptive temperature through restart ensembles across temperatures and entropy feedback, generalization of the diversity feature to output embeddings and intermediate-layer representations, richer structured artifacts, and tasks and generators on which diversity is actively required for quality.

Code and preregistrations: \url{https://github.com/o-ankomochi-o/T-GADE}.

\section*{Acknowledgment}
This work was supported by JSPS KAKENHI Grant Number JP23K11252.

\appendices
\section{Terminology and Symbols}\label{sec:names}
Table~\ref{tab:terms} lists the terms of this paper and Table~\ref{tab:symbols} the symbols.
\begin{table*}[t]\centering\footnotesize
\caption{Terminology.}\label{tab:terms}
\begin{tabular}{@{}l >{\raggedright\arraybackslash}p{15.2cm}@{}}\toprule
Term & Meaning\\\midrule
artifact & An LLM output with explicit structure; the substance of an individual. In this implementation, an explanation paired with code\\
energy & The value a task-specific evaluator assigns to an individual; lower is better. In this implementation,~\eqref{eq:energy}\\
feature & A vector $\phi_\ell$ that measures differences between artifacts; designed per problem\\
population diversity & The spread $H$ of a population computed from the features; in this paper the log-determinant of the Gram matrix\\
free energy & $F_T=\langle E\rangle-TH$; the quantity that survival selection minimizes\\
selection temperature & The weight $T$ of diversity; distinct from the LLM sampling temperature\\
occupancy & The rule for the number of individuals per level; Fermi-type permits at most one, Bose-type permits repetition\\
level & The equivalence class in which occupancy is counted; in this implementation the objective value rounded to five decimals\\
update scheme & Generational is short for the generational replacement model, steady-state for the steady-state model; the former replaces the population as a whole, the latter one or a few individuals per step, one in T-GADE\\
operator & An LLM operator that generates children\\
HoxLM & An LLM that attempts self-repair after evaluation and before survival selection; the HoxLM of this implementation is called the host\\
sealed instances & Evaluation instances never read during search; c100 and c500\\
budget point & The point at which the best training individual among registrations within $b$ total calls is evaluated\\
excess ratio & The ratio of the bin-count excess to the lower bound, excess for short; lower is better\\\bottomrule
\end{tabular}\end{table*}
\begin{table}[t]\centering\footnotesize
\caption{Symbols.}\label{tab:symbols}
\begin{tabular}{@{}l >{\raggedright\arraybackslash}p{6.2cm}@{}}\toprule
Symbol & Meaning\\\midrule
$x$, $y$ & individuals\\
$E(x)$ & energy\\
$\phi_\ell(x)$ & feature vector ($\ell=1,\ldots,m$)\\
$X_\ell$, $G_\ell$ & feature matrix and Gram matrix\\
$\varepsilon$, $I_n$ & regularization term and identity matrix\\
$H$, $F_T$ & population diversity and free energy\\
$T$ & selection temperature\\
$\mathcal{P}$, $Q$ & population and candidate pool\\
$N$, $n$ & target population size and current population size\\
$\Delta F_T$ & free-energy increment of adding a candidate\\
$d_\ell$, $R_\ell$, $c_\ell$, $v_\ell$ & Schur complement and auxiliary quantities of Appendix~\ref{sec:schur}\\
$r(x)$ & training excess ratio\\
$B_j(x)$, $L_j$ & bin count and volume lower bound on instance $j$\\
$b$ & total call budget\\
$C$, $a$, $\delta$ & bin capacity, available capacity before placement, residual capacity after placement\\
$\mathcal{F}$, $a_j$ & feasible bins and their available capacities\\
$i$, $s$, $\pi$ & position among feasible bins, bin score, score factor\\
$\rho$, $p$ & Spearman rank correlation and $p$ value of a test\\\bottomrule
\end{tabular}\end{table}

\section{Call Accounting and Endpoints}
Each admitted LLM request consumes one budget unit, including requests whose outputs cannot be parsed or evaluated; transport retries are counted separately. In the steady-state loop, operator calls are capped at 1,600 (3,200 for EoH) and the host is called at most once per child, so the recorded totals of the host-equipped runs lie between 3,175 and 3,197, and the EoH runs end at 3,200. In the generational loop each generation plans 9 children and 9 host calls; a failed operator call leaves a vacancy that is recorded, not refilled, and the recorded totals lie between 3,158 and 3,185. Evaluations of children before and after the host step and of registered individuals execute code without LLM calls. The budget-point selection uses only training values, and sealed scoring runs once after all runs have finished. In the reach-speed comparison the host-equipped runs ended at 792 to 799 total calls, that is, 400 operator calls plus fewer than 400 host calls. The budget point at $b$ sums, for each registration, the operator and host requests started before it, and admits only registrations with a sum of at most $b$. The scoring script used for the reported results counts a run as having reached budget $b$ when its total calls reach $b-1$, that is, with one call of slack. The two runs ending at 799 calls (seeds 86003 and 86007) are therefore admitted at $b=800$. This does not affect the preregistered primary endpoint at 400 total calls, and the last column of Table~\ref{tab:speed} is the endpoint of at most 800 calls.

\section{Schur-Complement Computation}\label{sec:schur}
For a selected prefix $\mathcal{P}$ of size $n\geq1$, let $R_\ell=(G_\ell(\mathcal{P})+\varepsilon I_n)^{-1}$ and $c_\ell(y)=X_\ell(\mathcal{P})\phi_\ell(y)$. The Schur complement of candidate $y$ is
\begin{equation}
d_\ell(y)=\|\phi_\ell(y)\|^2+\varepsilon-c_\ell(y)^\top R_\ell c_\ell(y).
\label{eq:schur}
\end{equation}
The unnormalized log-determinant gain of block $\ell$ is $\log d_\ell(y)$, and the normalized diversity increment is
\begin{equation}
H(\mathcal{P}\uplus\{y\})-H(\mathcal{P})=\frac{\sum_{\ell=1}^{m}\log d_\ell(y)-H(\mathcal{P})}{n+1}.\label{eq:diversity-increment}
\end{equation}
Together with the mean-energy increment $(E(y)-\langle E\rangle_{\mathcal{P}})/(n+1)$, this gives~\eqref{eq:increment}. After selecting $y$, write $v_\ell=R_\ell c_\ell(y)$ and update
\begin{equation}
R'_\ell=\begin{bmatrix}
R_\ell+v_\ell v_\ell^\top/d_\ell&-v_\ell/d_\ell\\
-v_\ell^\top/d_\ell&1/d_\ell
\end{bmatrix}.\label{eq:inverse-update}
\end{equation}
Regularization makes each Gram matrix positive definite, so $d_\ell>0$ in exact arithmetic. Given feature inner products, each increment takes $O(n^2)$ operations per block.

\bibliographystyle{IEEEtran}
\bibliography{refs}

\begin{thebibliography}{10}
\providecommand{\url}[1]{#1}
\csname url@samestyle\endcsname
\providecommand{\newblock}{\relax}
\providecommand{\bibinfo}[2]{#2}
\providecommand{\BIBentrySTDinterwordspacing}{\spaceskip=0pt\relax}
\providecommand{\BIBentryALTinterwordstretchfactor}{4}
\providecommand{\BIBentryALTinterwordspacing}{\spaceskip=\fontdimen2\font plus
\BIBentryALTinterwordstretchfactor\fontdimen3\font minus
  \fontdimen4\font\relax}
\providecommand{\BIBforeignlanguage}[2]{{%
\expandafter\ifx\csname l@#1\endcsname\relax
\typeout{** WARNING: IEEEtran.bst: No hyphenation pattern has been}%
\typeout{** loaded for the language `#1'. Using the pattern for}%
\typeout{** the default language instead.}%
\else
\language=\csname l@#1\endcsname
\fi
#2}}
\providecommand{\BIBdecl}{\relax}
\BIBdecl

\bibitem{elm2024}
J.~Lehman, J.~Gordon, S.~Jain, K.~Ndousse, C.~Yeh, and K.~O. Stanley,
  ``Evolution through large models,'' in \emph{Handbook of Evolutionary Machine
  Learning}.\hskip 1em plus 0.5em minus 0.4em\relax Springer, 2024, pp.
  331--366.

\bibitem{evoprompt2024}
\BIBentryALTinterwordspacing
Q.~Guo, R.~Wang, J.~Guo, B.~Li, K.~Song, X.~Tan, G.~Liu, J.~Bian, and Y.~Yang,
  ``Connecting large language models with evolutionary algorithms yields
  powerful prompt optimizers,'' in \emph{International Conference on Learning
  Representations}, 2024. [Online]. Available:
  \url{https://proceedings.iclr.cc/paper_files/paper/2024/hash/9156b0f6dfa9bbd18c79cc459ef5d61c-Abstract-Conference.html}
\BIBentrySTDinterwordspacing

\bibitem{promptbreeder2024}
\BIBentryALTinterwordspacing
C.~Fernando, D.~S. Banarse, H.~Michalewski, S.~Osindero, and
  T.~Rockt{\"a}schel, ``Promptbreeder: Self-referential self-improvement via
  prompt evolution,'' in \emph{Proceedings of the 41st International Conference
  on Machine Learning}, ser. Proceedings of Machine Learning Research, vol.
  235, 2024, pp. 13\,481--13\,544. [Online]. Available:
  \url{https://proceedings.mlr.press/v235/fernando24a.html}
\BIBentrySTDinterwordspacing

\bibitem{mori1995}
N.~Mori, J.~Yoshida, H.~Tamaki, H.~Kita, and Y.~Nishikawa, ``A thermodynamical
  selection rule for the genetic algorithm,'' in \emph{Proceedings of the 1995
  IEEE International Conference on Evolutionary Computation}, vol.~1, 1995, pp.
  188--192.

\bibitem{dejong1993gap}
K.~A. De~Jong and J.~Sarma, ``Generation gaps revisited,'' in \emph{Foundations
  of Genetic Algorithms}.\hskip 1em plus 0.5em minus 0.4em\relax Morgan
  Kaufmann, 1993, vol.~2, pp. 19--28.

\bibitem{satoh1996mgg}
H.~Satoh, M.~Yamamura, and S.~Kobayashi, ``Minimal generation gap model for
  {GAs} considering both exploration and exploitation,'' in \emph{Proceedings
  of the 4th International Conference on Soft Computing (IIZUKA'96)}, Iizuka,
  Japan, 1996, pp. 494--497.

\bibitem{sato1997gam}
H.~Sato, I.~Ono, and S.~Kobayashi, ``A new generation alternation model of
  genetic algorithms and its assessment,'' \emph{Journal of the Japanese
  Society for Artificial Intelligence}, vol.~12, no.~5, pp. 734--744, 1997, in
  Japanese.

\bibitem{eoh2024}
\BIBentryALTinterwordspacing
F.~Liu, X.~Tong, M.~Yuan, X.~Lin, F.~Luo, Z.~Wang, Z.~Lu, and Q.~Zhang,
  ``Evolution of heuristics: Towards efficient automatic algorithm design using
  large language model,'' in \emph{Proceedings of the 41st International
  Conference on Machine Learning}, ser. Proceedings of Machine Learning
  Research, vol. 235, 2024, pp. 32\,201--32\,223. [Online]. Available:
  \url{https://proceedings.mlr.press/v235/liu24bs.html}
\BIBentrySTDinterwordspacing

\bibitem{akiba2025merge}
T.~Akiba, M.~Shing, Y.~Tang, Q.~Sun, and D.~Ha, ``Evolutionary optimization of
  model merging recipes,'' \emph{Nature Machine Intelligence}, vol.~7, pp.
  195--204, 2025.

\bibitem{funsearch2024}
B.~Romera-Paredes, M.~Barekatain, A.~Novikov, M.~Balog, M.~P. Kumar, E.~Dupont,
  F.~J.~R. Ruiz, J.~S. Ellenberg, P.~Wang, O.~Fawzi, P.~Kohli, and A.~Fawzi,
  ``Mathematical discoveries from program search with large language models,''
  \emph{Nature}, vol. 625, pp. 468--475, 2024.

\bibitem{reevo2024}
H.~Ye, J.~Wang, Z.~Cao, F.~Berto, C.~Hua, H.~Kim, J.~Park, and G.~Song,
  ``{ReEvo}: Large language models as hyper-heuristics with reflective
  evolution,'' in \emph{Advances in Neural Information Processing Systems},
  vol.~37, 2024, pp. 43\,571--43\,608.

\bibitem{llamea2025}
N.~van Stein and T.~B{\"a}ck, ``{LLaMEA}: A large language model evolutionary
  algorithm for automatically generating metaheuristics,'' \emph{IEEE
  Transactions on Evolutionary Computation}, vol.~29, no.~2, pp. 331--345,
  2025.

\bibitem{hsevo2025}
{Pham Vu Tuan Dat}, {Long Doan}, and {Huynh Thi Thanh Binh}, ``{HSEvo}:
  Elevating automatic heuristic design with diversity-driven harmony search and
  genetic algorithm using {LLMs},'' in \emph{Proceedings of the AAAI Conference
  on Artificial Intelligence}, vol.~39, no.~25, 2025, pp. 26\,931--26\,938.

\bibitem{mctsahd2025}
\BIBentryALTinterwordspacing
Z.~Zheng, Z.~Xie, Z.~Wang, and B.~Hooi, ``{Monte Carlo} tree search for
  comprehensive exploration in {LLM}-based automatic heuristic design,'' in
  \emph{Proceedings of the 42nd International Conference on Machine Learning},
  ser. Proceedings of Machine Learning Research, vol. 267, 2025, pp.
  78\,338--78\,373. [Online]. Available:
  \url{https://proceedings.mlr.press/v267/zheng25o.html}
\BIBentrySTDinterwordspacing

\bibitem{eohs2026}
F.~Liu, Y.~Liu, Q.~Zhang, X.~Tong, and M.~Yuan, ``{EoH-S}: Evolution of
  heuristic set using {LLMs} for automated heuristic design,'' in
  \emph{Proceedings of the AAAI Conference on Artificial Intelligence},
  vol.~40, no.~43, 2026, pp. 37\,090--37\,098.

\bibitem{llmlns2025}
\BIBentryALTinterwordspacing
H.~Ye, H.~Xu, A.~Yan, and Y.~Cheng, ``Large language model-driven large
  neighborhood search for large-scale {MILP} problems,'' in \emph{Proceedings
  of the 42nd International Conference on Machine Learning}, ser. Proceedings
  of Machine Learning Research, vol. 267, 2025, pp. 72\,131--72\,180. [Online].
  Available: \url{https://proceedings.mlr.press/v267/ye25j.html}
\BIBentrySTDinterwordspacing

\bibitem{novelty2011}
J.~Lehman and K.~O. Stanley, ``Abandoning objectives: Evolution through the
  search for novelty alone,'' \emph{Evolutionary Computation}, vol.~19, no.~2,
  pp. 189--223, 2011.

\bibitem{mapelites2015}
\BIBentryALTinterwordspacing
J.-B. Mouret and J.~Clune, ``Illuminating search spaces by mapping elites,''
  2015. [Online]. Available: \url{https://arxiv.org/abs/1504.04909}
\BIBentrySTDinterwordspacing

\bibitem{incontextqd2024}
\BIBentryALTinterwordspacing
B.~Lim, M.~Flageat, and A.~Cully, ``Large language models as in-context {AI}
  generators for quality-diversity,'' 2024. [Online]. Available:
  \url{https://arxiv.org/abs/2404.15794}
\BIBentrySTDinterwordspacing

\bibitem{dpp2012}
A.~Kulesza and B.~Taskar, ``Determinantal point processes for machine
  learning,'' \emph{Foundations and Trends in Machine Learning}, vol.~5, no.
  2--3, pp. 123--286, 2012.

\bibitem{mori1996}
N.~Mori, H.~Kita, and Y.~Nishikawa, ``Adaptation to a changing environment by
  means of the thermodynamical genetic algorithm,'' in \emph{Parallel Problem
  Solving from Nature---PPSN IV}, ser. Lecture Notes in Computer Science, vol.
  1141.\hskip 1em plus 0.5em minus 0.4em\relax Springer, 1996, pp. 513--522.

\bibitem{mori1998}
------, ``Adaptation to a changing environment by means of the feedback
  thermodynamical genetic algorithm,'' in \emph{Parallel Problem Solving from
  Nature---PPSN V}, ser. Lecture Notes in Computer Science, vol. 1498.\hskip
  1em plus 0.5em minus 0.4em\relax Springer, 1998, pp. 149--158.

\bibitem{gptoss2025}
{OpenAI}, ``{gpt-oss-120b} \& {gpt-oss-20b} model card,'' arXiv:2508.10925,
  2025.

\bibitem{vllm2023}
W.~Kwon, Z.~Li, S.~Zhuang, Y.~Sheng, L.~Zheng, C.~H. Yu, J.~Gonzalez, H.~Zhang,
  and I.~Stoica, ``Efficient memory management for large language model serving
  with {PagedAttention},'' in \emph{Proceedings of the 29th Symposium on
  Operating Systems Principles (SOSP '23)}.\hskip 1em plus 0.5em minus
  0.4em\relax ACM, 2023, pp. 611--626.

\end{thebibliography}
\end{document}